\documentclass[10pt,twocolumn,letterpaper]{article}

\usepackage{wacv}              % To produce the CAMERA-READY version
\usepackage{graphicx}
\usepackage{amsmath}
\usepackage{amsthm}

\usepackage{amssymb}
\usepackage{booktabs}

\usepackage{cite}
\usepackage{amsmath,amssymb,amsfonts}
\usepackage{algorithmic}
\usepackage{graphicx}
\usepackage{algorithm,algorithmic}
\usepackage{hyperref}
\hypersetup{hidelinks}
\usepackage{textcomp}
 \usepackage{multirow}  
 
\usepackage{float}
\newtheorem{Lemma}{Lemma}
\newtheorem{theorem}{Theorem}
\usepackage[capitalize]{cleveref}
\crefname{section}{Sec.}{Secs.}
\Crefname{section}{Section}{Sections}
\Crefname{table}{Table}{Tables}
\crefname{table}{Tab.}{Tabs.}

\def\wacvPaperID{789} % *** Enter the WACV Paper ID here
\def\confName{WACV}
\def\confYear{2025}

\begin{document}

% \title{\LaTeX\ Author Guidelines for \confName~Proceedings}
% \title{CIG: Contrastive Integrated Gradients for Feature Attribution in Whole Slide Images}
% \title{CIG: Contrastive Integrated Gradients for Feature Attribution in Explaining Whole Slide Image Classification}
\title{Contrastive Integrated Gradients: A Feature Attribution-Based Method for Explaining Whole Slide Image Classification}
\author{
Anh Mai Vu$^{1}$ \quad
Tuan L. Vo$^{2}$ \quad
Ngoc Lam Quang Bui$^{3}$ \quad
Nam Nguyen Le Binh$^{4}$\\
Akash Awasthi$^{1}$ \quad
Huy Quoc Vo$^{1}$ \quad
Thanh\mbox{-}Huy Nguyen$^{5}$ \quad
Zhu Han$^{1}$\\
Chandra Mohan$^{1}$ \quad
Hien Van Nguyen$^{1}$\\[0.2cm]
$^{1}$ECE Department, University of Houston\\ $^{2}$Information Technology, HCMC University of Technology and Education\\
$^{3}$VN-UK Institute for Research and Executive Education, The University of Da Nang\\
$^{4}$Ho Chi Minh City University of Science, Vietnam National University\\
$^{5}$Computational Biology Department, Carnegie Mellon University\\[0.1cm]
{\tt\small mvu9@cougarnet.uh.edu}
}

% \and
\maketitle

%%%%%%%%% ABSTRACT
\begin{abstract}
Interpretability is essential in Whole Slide Image (WSI) analysis for computational pathology, where understanding model predictions helps build trust in AI-assisted diagnostics. While Integrated Gradients (IG) and related attribution methods have shown promise, applying them directly to WSIs introduces challenges due to their high-resolution nature. These methods capture model decision patterns but may overlook class-discriminative signals that are crucial for distinguishing between tumor subtypes. In this work, we introduce \textbf{C}ontrastive \textbf{I}ntegrated \textbf{G}radients (\textbf{CIG}), a novel attribution method that enhances interpretability by computing contrastive gradients in logit space. First, CIG highlights class-discriminative regions by comparing feature importance relative to a reference class, offering sharper differentiation between tumor and non-tumor areas. Second, CIG satisfies the axioms of integrated attribution, ensuring consistency and theoretical soundness. Third, we propose two attribution quality metrics, MIL-AIC and MIL-SIC, which measure how predictive information and model confidence evolve with access to salient regions, particularly under weak supervision. We validate CIG across three datasets spanning distinct cancer types: CAMELYON16 (breast cancer metastasis in lymph nodes), TCGA-RCC (renal cell carcinoma), and TCGA-Lung (lung cancer). Experimental results demonstrate that CIG yields more informative attributions both quantitatively, using MIL-AIC and MIL-SIC, and qualitatively, through visualizations that align closely with ground truth tumor regions, underscoring its potential for interpretable and trustworthy WSI-based diagnostics.

% The implementation of our proposed model and attribution method is available at the following anonymized repository: \url{https://anonymous.4open.science/r/CIG-2822}

% \textit{Code for the proposed Contrastive Integrated Gradients (CIG) method is available at} \url{https://github.com/maianhpuco/CIG.git} 
 
\end{abstract} 

%%%%%%%%% BODY TEXT
\section{Introduction} \label{sec:introduction}

\begin{figure*}[t]
    \centering
    \includegraphics[width=0.8\textwidth]{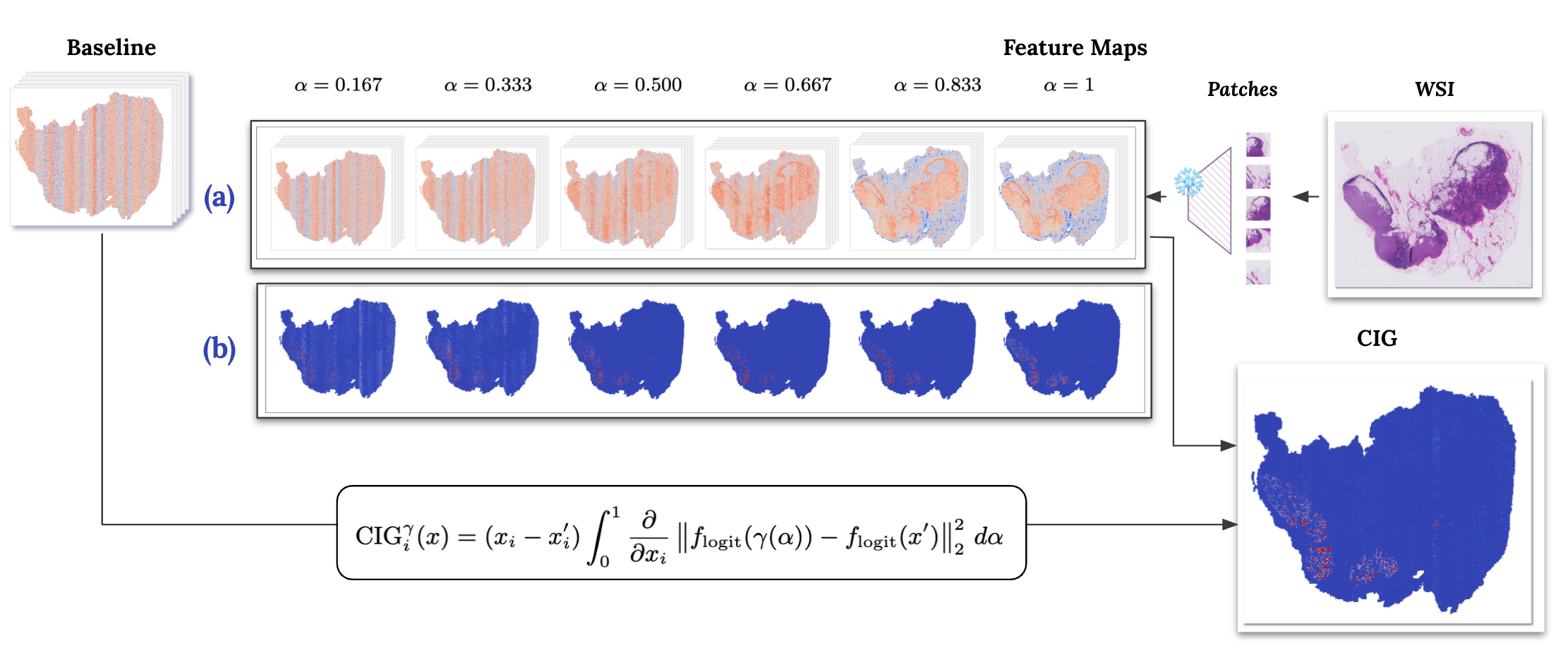}
    \caption{\textbf{Overview of Contrastive Integrated Gradients (CIG).} Given a whole-slide image (WSI), patch-level features are extracted and compared to a baseline sampled from non-tumor regions. An interpolated path \( \gamma(\alpha) = x + \alpha(x' - x) \) is constructed between the input \( x \) and the baseline \( x' \). CIG computes attributions by integrating the gradients of the squared logit difference along this path, where \( f_{\text{logit}}(\cdot) \) denotes the model’s logit output and \( \| \cdot \|_2 \) is the Euclidean norm. Row (a) shows interpolated features at different \( \alpha \) values (\( \alpha = 0.167 \) to \( 1 \)). Row (b) illustrates how contrastive gradients evolve with increasing \( \alpha \), indicating the sensitivity of each feature at each interpolation step. The full attribution is computed by summing the gradients across all \( \alpha \) values and multiplying by the input difference \( x - x' \). The final heatmap (bottom right) shows the CIG attribution result, indicating which regions most strongly influence the model’s decision relative to the baseline.}
    \label{fig:overview}
\end{figure*}  

Pathology plays a vital role in modern medicine, providing the foundation for accurate diagnosis and prognosis across a wide spectrum of diseases. At the core of pathological assessment lies the rigorous examination of whole slide images (WSIs), where pathologists must carefully analyze cellular composition and tissue architecture to make reliable assessments, such as identifying tumors. This process, while fundamental to patient care, faces significant challenges, including substantial time requirements and notable inter-observer variability among specialists interpreting the same visual data. Recent advances in artificial intelligence and deep learning offer promising solutions to these challenges by enabling the development of computational tools capable of pre-segmenting regions of interest within WSIs. These tools can help guide pathologists’ attention to critical areas, potentially reducing diagnostic inconsistencies and improving workflow efficiency. 

Recent advances in WSI analysis have introduced deep learning models for both classification and segmentation. Classification methods often follow the Multiple Instance Learning (MIL) paradigm, including ABMIL \cite{ilse2018attention-abmil}, CLAM \cite{clam_lu2021data}, TransMIL \cite{shao2021transmil}, HIPT \cite{chen2022scaling-hipt}, DSMIL \cite{li2021dual-dsmil}, and more recent models like CAMIL \cite{camil_fourkioti2024camil}, DTFD-MIL \cite{dtfd_zhang_dtfd-mil_2022}, SNUFFY \cite{snuffy_afarinia2024snuffyefficientslideimage}, and DGMIL \cite{qu2022dgmil}. Segmentation approaches such as generalized U-Net variants \cite{wsi_seg_khened2021generalized}, DHUNet \cite{unet_seg_3_wang2023dhunet}, and multi-scale methods \cite{seg2_XU2016214} focus on delineating tissue regions. While these models perform well, they provide limited insight into the histological patterns behind their predictions. This has motivated interest in attribution methods that help connect model outputs to relevant tissue regions for improved interpretability.

According to \cite{benou2025showandtell}, explainable AI (XAI) methods in computer vision can be grouped along two dimensions: \textit{heatmap-based vs. concept-based}, where the former highlights \textit{where} a model focuses and the latter explains \textit{what} it relies on, and \textit{post-hoc vs. ante-hoc}, where post-hoc methods generate explanations after training (e.g., Labo \cite{yang2023languageinabottle}, TCAV \cite{kim2018interpretability}) and ante-hoc methods integrate interpretability into the model (e.g., Concept Bottleneck Models \cite{koh2020concept, huang2024concept}). In computational pathology, two complementary directions have emerged: ante-hoc methods such as SI-MIL \cite{kapse2024si-mil}, Concept-MIL \cite{sun2025label}, and GEKCO \cite{kapse2025gecko}, which provide global interpretability through text–image alignment or handcrafted features, and post-hoc attribution methods, which generate localized heatmaps of tumor-associated regions within WSIs. Our work advances post-hoc methods by improving the fidelity and accuracy of visualizations in computational pathology.

Ensuring interpretability in deep learning models is crucial for WSI analysis in computational pathology; it is crucial for building trust in clinical applications. Gradient-based methods such as Gradient × Input \cite{grad_x_input_ancona2017towards}, EG \cite{eg_erion2020improvingperformancedeeplearning}, Guided IG \cite{gig_kapishnikov_guided_2021}, Grad-CAM \cite{grad_cam_selvaraju2017grad}, IG \cite{ig_sundararajan_axiomatic_2017}, and IDG \cite{idg_walker_integrated_2023} are commonly used for explainability. However, these typically operate in image space and may be less effective in models that rely on learned embeddings.

% To address this, we propose \textbf{Contrastive Integrated Gradients}, a logit-space attribution method that captures contrastive information between inputs and reference baselines, aiming to improve interpretability in WSI classification. 

% \textbf{SHAP (SHapley Additive exPlanations).} 
Beyond gradient-based approaches, model-agnostic methods like SHAP offer a complementary perspective on feature attribution. SHAP \cite{shap} is a widely used, theoretically grounded approach to feature attribution that satisfies key axioms for fair and consistent explanations. Approaches such as LIME \cite{lime2016ribeiro} and DeepLIFT \cite{DeepLIFT2017Shrikumar} offer valuable tools for local interpretability and backpropagation-based attribution, respectively. SHAP builds upon these ideas by providing a unified framework grounded in Shapley values \cite{shapley}. To improve scalability, Kernel SHAP \cite{shap} introduces a kernel-based weighting strategy with linear regression over sampled feature subsets, while Deep SHAP \cite{shap} extends SHAP to deep models by incorporating DeepLIFT’s backpropagation mechanism. While these advancements enhance interpretability, optimizing SHAP for large-scale medical image analysis needs further improvement, particularly in balancing computational efficiency with robust feature attribution.  

% \textbf{Attribution Method.} 
% Attribution methods assign importance scores to input features based on their contributions to a model's output. Integrated Gradients (IG) \cite{IG} addresses limitations of saliency maps, such as sensitivity to noise and gradient saturation. IG computes attributions by integrating gradients along a straight-line path from a baseline input to the actual input, satisfying properties like sensitivity, completeness, linearity, and implementation invariance. The choice of baseline, representing the absence or neutral state of the input, is critical. Common baselines include zero vectors, blurred images, or random noise, each with limitations, particularly in domains where zero has a specific meaning or does not represent feature absence. EG \cite{EG} extends IG by averaging attributions over the distribution of baselines, improving robustness but increasing computational cost. IDG \cite{idg_walker_integrated_2023} refines IG by computing attributions at decision-critical regions instead of following a straight-line path from a baseline. By focusing on points where the model makes decisions, IDG reduces attribution noise but introduces computational overhead. Recent advancements focus on improving precision and robustness. Methods like BIG \cite{BIG} refine baseline selection, while AGI \cite{AGI} explores nonlinear trajectories for more accurate attributions. 

% However, these approaches add computational complexity, highlighting the challenge of balancing interpretability, efficiency, and reliability in feature attribution.

Attribution methods assign importance scores to input features based on their contributions to a model’s output. Integrated Gradients (IG) \cite{IG} improves upon basic saliency maps by integrating gradients along a straight-line path from a baseline to the input, and satisfies desirable properties such as sensitivity, completeness, and implementation invariance. However, its effectiveness heavily depends on the choice of baseline, which can be ambiguous in domains like pathology. Variants like EG \cite{EG} improve robustness by averaging over multiple baselines, while IDG \cite{idg_walker_integrated_2023} focuses on decision-critical points rather than path integration. Other extensions such as BIG \cite{BIG} and AGI \cite{AGI} refine baselines or explore nonlinear paths for better precision, though often at higher computational cost. 
Despite these advances, most attribution methods operate in image space and may highlight visually salient but class-irrelevant features, limiting their interpretability for WSI classification. To address this, we propose \textbf{Contrastive Integrated Gradients}, a logit-space attribution method that captures contrastive information between inputs and reference baselines. By emphasizing class-discriminative features those that distinguish one class from others CIG aims to provide more meaningful and focused attributions in computational pathology. In summary, our contributions are as follows: 
\begin{itemize}
    \item We introduce CIG, a novel Integrated Gradients-based path attribution method that enhances interpretability by capturing contrastive information in logit space. By emphasizing class-discriminative gradients, CIG offers clearer insights into model predictions.
    \item We demonstrate that CIG satisfies the axioms of integrated attribution methods, ensuring consistent and principled feature-space attributions that enhance the reliability and interpretability of model explanations. 
    \item We incorporate CIG into a Multiple Instance Learning (MIL) framework from 
    \cite{kapishnikov2019xrai} for WSI classification, enabling feature-space attributions and effective visualization of tumor regions. To evaluate attribution quality in this weakly supervised setting, we adapt the PIC framework and introduce two new metrics, \textit{MIL-AIC} and \textit{MIL-SIC}, to assess class consistency and confidence progression.
\end{itemize}

\section{Background}\label{sec:background}
% \section{Background}
\textbf{Path Integrated Gradients (PathIG)} \cite{ig_sundararajan_axiomatic_2017}, based on Shapley Theory \cite{shap}, compute feature importance by integrating gradients along a predefined path from a baseline \( x' \) and an input \( x \) in the input space. %This method aggregates gradients along points on a path between a baseline \( x' \) and an input \( x \). 
The feature attribution is computed as:   
\begin{equation}
    \text{PathIG}_i^\gamma (x) = \int_0^1 \frac{\partial F(\gamma(\alpha))}{\partial \gamma_i(\alpha)} \cdot \frac{\partial \gamma_i(\alpha)}{\partial \alpha} \, d\alpha.
\end{equation}  
\noindent %Given an input \( x \) and a baseline \( x' \), 
Here, \( F: \mathbb{R}^n \to [0,1] \) represents a deep network, \( \gamma(\alpha): [0,1] \to \mathbb{R}^n \) is a path function, and $\frac{\partial F(\gamma(\alpha))}{\partial \gamma_i(\alpha)} = \nabla_i F(\gamma(\alpha))$ is the gradients of the model output with respect to the \( i \)th feature of input.

\textbf{Integrated Gradients (IG)} is a special case of PathIG when $\gamma(\alpha)$ is a straightline path, i.e, $\gamma(\alpha) = x' + \alpha (x - x')$ for \( \alpha \in [0,1] \). Since $\frac{\partial \gamma_i(\alpha)}{\partial \alpha} = (x_i - x'_i)$, the formula of IG is given by:
\begin{equation}
    \text{IG}_i(x) = (x_i - x'_i)\int_0^1 \frac{\partial F(x' + \alpha (x - x'))}{\partial x_i} \, d\alpha.
\end{equation} 
This integral can be approximated using a Riemann sum: 
\begin{equation}
    \text{IG}_i^{\text{approx}}(x) = (x_i - x'_i) \times \sum_{k=1}^{m} \frac{\partial F(x' + \frac{k}{m} (x - x'))}{\partial x_i} \times \frac{1}{m},
\end{equation} 
where \( m \) is the number of steps in the approximation. This summation efficiently estimates the integral and can be computed in a loop over \( m \) points along the path. 

\textbf{Axiomatic Properties of Integrated Gradients.} IG adheres to several important axioms of attribution methods. First, IG satisfies \textbf{Completeness} which ensures that attributions sum to the difference in model output between the input \( x \) and the baseline \( x' \): 
\begin{equation*}
    \sum_{i=1}^{n} \text{IG}_i(x) = F(x) - F(x'),
\end{equation*}
where $n$ is the number of features. Additionally, IG also upholds the \textbf{Sensitivity axiom}, which states that if only one feature differs between the input and the baseline, and this difference causes a change in the model output, then all of the attributions should go to that feature. Furthermore, IG maintains the \textbf{Implementation Invariance}, guaranteeing that attributions remain consistent across functionally equivalent networks since they rely solely on the gradients of the function represented by the network.

% \textbf{$\text{IG}^2$ Attribution.} $\text{IG}^2$ \cite{ig2_zhuo_2024} extends IG by incorporating model representation differences rather than just input feature differences. The $\text{IG}^2$ attribution for the \( i \)-th feature is computed as:

% \begin{equation}
%     \phi_i^{IG^2} = \sum_{k=0}^{m-1} \frac{\partial f(\gamma^G(k/m))}{\partial x_i}
%     \times 
%     \frac{\partial \|\tilde{f}(\gamma^G(k/m)) - \tilde{f}(x^r)\|_2^2}{\partial x_i} \times \frac{\eta}{W_j},
% \end{equation}

% where \( f(\cdot) \) represents the model’s prediction, and \( \tilde{f}(\cdot) \) denotes its internal representation layer, \( \gamma^G(k/m) \) defines an interpolation path between the input and the reference. The first term captures the explicand's gradients, measuring how \( x_i \) influences the model output. The second term, known as the counterfactual gradients, quantifies how the model’s internal representation responds to input changes.  

% Inspired by $\text{IG}^2$, we apply this technique to WSI by computing gradients on logit differences instead of raw features. Unlike $\text{IG}^2$, which requires a predefined representation layer, our method directly uses logits, making it more convenient and adaptable for WSI classification.

\section{Methodology}
% \textcolor{red}{Explain that the input in is paper is the Features From the Pretrained model (not pixel space)}
\subsection{Contrastive Integrated Gradients} 
We introduce the Contrastive Integrated Gradients, a novel path-based attribution method designed to highlight key distinguishing features, particularly in differentiating tumors from non-tumor regions.

\textbf{Definition.} Given an input \( x \) and a reference \( x' \), the attribution for the \( i \)th feature is computed as:  
 \begin{equation}
    \text{CIG}_i^{\gamma}(x) = (x_i - x_i') \int_{0}^{1} \frac{\partial \| f_{\text{logit}}(\gamma(\alpha)) - f_{\text{logit}}(x')\|_2^2}{\partial x_i} d\alpha,
\end{equation}
where \( \gamma(\alpha) \) defines a straight line path between $x$ and $x'$, \( f_{\text{logit}}(\cdot) \) represents the model’s logit layer, and $\|\cdot\|_2$ denotes the $\ell_2$ norm.

Unlike IG, which attributes importance using gradients from the input, CIG captures contrastive information by computing the difference between the logits of interpolated inputs and a baseline reference \( x' \). CIG integrates gradients along a path in logit space, measuring how the model’s decision boundary evolves relative to the reference. This allows the method to capture contrastive information, making it effective in distinguishing between categories.

The following Lipschitz‐based bound shows how large a single‐feature CIG can be, given the input change and the model’s smoothness.

\begin{Lemma}
\textbf{CIG Bound.} 
Assume \(f_{\mathrm{logit}}:\mathbb{R}^{n}\!\to\!\mathbb{R}^{m}\) be \(L\)-\emph{Lipschitz} in the Euclidean norm, i.e,  
\begin{equation*}
    \|f_{\mathrm{logit}}(x)-f_{\mathrm{logit}}(x')\|_2
      \;\le\;L\,\|x-x'\|_2
      \qquad\forall\,x,x'\in\mathbb{R}^{n}.
\end{equation*}
Then for every feature \(i\in\{1,\dots,n\}\),
\begin{equation}
    \bigl|\operatorname{CIG}^{\gamma}_{i}(x)\bigr|
      \;\le\;\tfrac{2}{3}\,L^{2}\,\|x-x'\|_2\,|x_i-x_i'|.
\end{equation}
\end{Lemma}
\begin{proof} 
Set $r(\alpha)=f_{\mathrm{logit}}\!\bigl(\gamma(\alpha)\bigr)-f_{\mathrm{logit}}(x'),  
 g(\alpha)=\|r(\alpha)\|_2^{2}.$
Because $g(\alpha)=r(\alpha)^{\!\top}r(\alpha)$,
\begin{align*}
    \frac{\partial g(\alpha)}{\partial x_i} &= 2\,r(\alpha)^{\!\top}\,\frac{\partial r(\alpha)}{\partial x_i}\\
    &=2\,r(\alpha)^{\!\top}\,\alpha\,
        \frac{\partial f_{\mathrm{logit}}}{\partial x_i}
        (\gamma(\alpha)).
\end{align*}
Because $f_{\text{logit}}$ is \(L\)-\emph{Lipschitz}, with $x=\gamma(\alpha)$:
\begin{equation*}
    \|r(\alpha)\|_2
      \le L\|\gamma(\alpha)-x' \|_2 \le L\,\alpha\|x-x'\|_2.
\end{equation*}
Moreover, since $\bigl\|\tfrac{\partial f_{\mathrm{logit}}}{\partial x_i}\bigr\|_2\le L$,
\begin{equation*}
    \Bigl|\tfrac{\partial g(\alpha)}{\partial x_i}\Bigr|
      \le 2\alpha\,(L\alpha\|x-x'\|_2)\,L
      =2L^{2}\alpha^{2}\|x-x'\|_2.
\end{equation*}
The CIG along~\(\gamma\) is:
\begin{equation*}
    \operatorname{CIG}^{\gamma}_{i}(x)
      =(x_i-x_i')
        \int_{0}^{1}
            \frac{\partial}{\partial x_i}\,g(\alpha)\,d\alpha.
\end{equation*}
Therefore,
\begin{align*}
    \bigl|\operatorname{CIG}^{\gamma}_{i}(x)\bigr|
   &\le |x_i-x_i'|\,
        \int_{0}^{1}2L^{2}\alpha^{2}\|x-x'\|_2\,d\alpha\\
   &\le\tfrac23\,L^{2}\,\|x-x'\|_2\,|x_i-x_i'|.
\end{align*}
\end{proof}
This bound quantitatively links a feature’s CIG attribution to how much that feature and the full input shift, scaled by the network’s Lipschitz smoothness.

\subsection{Axiomatic Properties of Contrastive Integrated Gradients} 
\begin{theorem}[Completeness Axiom]
The total attribution must equal the change in the model’s representation in logit space. Mathematically, this is given by: 
\begin{equation}
    \sum_{i=1}^{n} \text{CIG}_i^{\gamma}(x) = D(x) - D(x'),
\end{equation}

where $D(x) = \left\|f_{\text{logit}}(x) - f_{\text{logit}}(x')\right\|_2^2.$
\end{theorem}
Since $D(x')=0$, we can generally rewrite the completeness axiom as follows:
\begin{equation}
    \sum_{i=1}^{n} \text{CIG}_i^\gamma(x) = \| f_{\text{logit}}(x) - f_{\text{logit}}(x') \|_2^2,
\end{equation}

\begin{proof}

Let the straight line path from \(x'\) to \(x\) be
\[
\gamma(\alpha)=x'+\alpha\,(x-x'),\qquad \alpha\in[0,1].
\]
Set $g(\alpha)=\|f_{\mathrm{logit}}\bigl(\gamma(\alpha)\bigr)-f_{\mathrm{logit}}(x')\|_2^2$.
Using the chain rule:
\begin{align*}
    \frac{d}{d\alpha}g(\alpha)&=\nabla g_\gamma(\gamma(\alpha))\cdot\frac{d\gamma(\alpha)}{d\alpha} \\
    &=\nabla g_\gamma(\gamma(\alpha))\cdot(x-x').
\end{align*}
Writing component-wise, we obtain:
\begin{equation}\label{eq:dg}
    \frac{d}{d\alpha}g(\alpha)
      =\sum_{i=1}^{n}(x_i-x_i')\,
        \frac{\partial}{\partial x_i}g(\alpha).
\end{equation}

By definition of the CIG,
\begin{equation*}
    \operatorname{CIG}^{\gamma}_{i}(x)
       =(x_i-x_i')\int_{0}^{1}
         \frac{\partial}{\partial x_i}
         g(\alpha)\,d\alpha.
\end{equation*}\\
Summing over \(i\) and exchanging sum with integral yields
\begin{align*}  \sum_{i=1}^{n}\operatorname{CIG}^{\gamma}_{i}(x)&=\sum_{i=1}^{n}(x_i-x_i')\int_{0}^{1}\frac{\partial}{\partial x_i}g(\alpha)\,d\alpha\\
    &=\int_{0}^{1}\sum_{i=1}^{n}(x_i-x_i')\frac{\partial}{\partial x_i}g(\alpha)\,d\alpha. 
\end{align*}  
Equation~\eqref{eq:dg} allows us to simplify the expression to:
\begin{align*}
\sum_{i=1}^{n}\operatorname{CIG}^{\gamma}_{i}(x)    &=\int_{0}^{1}dg(\alpha)=g(1)-g(0)\\
    &=\bigl\|f_{\mathrm{logit}}(x)-f_{\mathrm{logit}}(x')\bigr\|_{2}^{2}.
\end{align*}
\end{proof}  
% \textcolor{red}{Summing over $n$ features, interchanging the sum and integral:
% \begin{equation}
%     \sum_{i=1}^{n} \text{CIG}_i^j(x) = \int_{0}^{1} \nabla \| f_{\text{logit}}(\gamma^G(j/k)) - f_{\text{logit}}(x') \|_2^2 \cdot (x - x') d\alpha.
% \end{equation} 
% This integral evaluates to $   \| f_{\text{logit}}(x) - f_{\text{logit}}(x') \|_2^2$, by the Fundamental Theorem of Calculus. Since the total attributions match the difference in logit outputs, CIG satisfies completeness. This ensures that the attributions accurately reflect the change in the model’s output. }

CIG satisfies the \textbf{Sensitivity Axiom} by assigning nonzero attributions when \( x_i \neq x_i' \) and \( f_{\text{logit}} \) depends on \( x_i \), as the gradients are generally nonzero along the path. Conversely, if \( f_{\text{logit}} \) is independent of \( x_i \), the gradients remain zero, resulting in zero attribution. Thus, CIG correctly distinguishes relevant and irrelevant features.

It also adheres to the \textbf{Implementation Invariance Axiom}, as its attributions are solely based on the model’s output function and gradients, ensuring consistency regardless of implementation specifics.

\subsection{Incorporating CIG into the WSI Analysis Pipeline}
\subsubsection{Overview of the MIL-based Classification Pipeline for Feature Attribution} 
In a typical WSI classification pipeline, each WSI is divided into patches at a fixed magnification level for deep learning processing. These patches are passed through a feature extractor, such as a Vision Transformer or CNN, to generate embeddings that capture meaningful information. The embeddings are treated as instances in an MIL setup, where each slide is a bag of patches. For a classification task, a WSI is labeled positive if any of its patches are positive (e.g., contain tumor). Since patch-level labels are unavailable, MIL classifiers use pooling mechanisms such as attention-based pooling to aggregate patch-level information and produce slide-level predictions.This enables weakly supervised learning while also identifying informative regions. Attribution methods like CIG then use gradients with respect to the feature embeddings each corresponding to a specific patch to highlight the regions that contribute most to the final prediction.

% In a typical WSI classification pipeline, each Whole Slide Image (WSI) is first divided into non-overlapping patches, often at a fixed magnification level (e.g., 20$\times$), to make processing feasible for deep learning models. These image patches are individually passed through a feature extractor, such as a pre-trained Vision Transformer (ViT) or convolutional neural network (CNN), to generate patch-level embeddings that capture relevant semantic and morphological information. These embeddings are then treated as input instances within a Multiple Instance Learning (MIL) framework, where each slide is modeled as a "bag" of instances (patches). Under the MIL assumption, the slide (bag) is labeled positive if at least one of its instances is positive (e.g., contains tumor), and negative otherwise. Because instance labels are not available during training, MIL classifiers typically aggregate the patch-level features using a permutation-invariant pooling mechanism, such as attention-based pooling, to produce a bag-level prediction. This setup allows models to learn slide-level classifications from weak supervision while implicitly localizing informative regions. The resulting gradients from this trained MIL classifier can then be leveraged by attribution methods such as Contrastive Integrated Gradients (CIG) to identify which patches contribute most to the model's decision.

\subsubsection{Design of the Attribution Baseline} 

Selecting a suitable reference input (or baseline) is critical in gradient-based attribution methods like Integrated Gradients (IG), as it defines the reference point for measuring feature importance. In natural images, common baselines include black, white, or average images. However, in our setting where inputs are bags of WSI patches, we found common strategies to be limited. A zero vector baseline can cause out-of-distribution issues, as it may carry unintended meaning in the embedding space. Using the dataset mean offers stability but introduces semantic bias, often favoring the dominant class. Sampling from the dataset distribution or using random patches suffers from similar problems, such as semantic leakage or reduced discriminativeness when input and baseline belong to the same class.

To overcome this, we use a \textbf{baseline from the opposite class}. For tumor-positive slides, we sample patches from non-tumor slides as the reference. This strategy better captures the meaningful differences that drive model predictions, resulting in clearer and more interpretable saliency maps. It benefits both standard IG and contrastive attribution methods by enhancing sensitivity to class-specific features and aligning more closely with the model’s decision boundaries.

\subsubsection{Attribution Quality Assessment } 
% \subsubsection{Evaluating Attribution Quality}

% Evaluating the quality of attribution maps is a critical yet challenging task, especially in the context of Whole Slide Image (WSI) analysis, where explanation ground truth is not usually available. In natural image settings, most evaluation metrics are based on the assumption that pixels with high attribution scores correspond to regions that contribute most to the model’s prediction for the target class~\cite{kapishnikov2019xrai, petsiuk2018rise}. 
Evaluating attribution quality is especially challenging in WSI analysis, where ground truth explanations are often unavailable. In natural image tasks, evaluation typically assumes that high-saliency pixels contribute most to model predictions~\cite{kapishnikov2019xrai, petsiuk2018rise}. 

% \textbf{Performance Information Curves (PICs)}, introduced by XRAI~\cite{kapishnikov2019xrai}, begin by heavily blurring the original image to remove most of its visual information. Pixels are then gradually reintroduced based on their saliency ranking. At each step, the image’s information level is estimated using compression-based entropy (e.g., WebP file size). The classifier is applied to each intermediate image, and performance metrics (e.g., accuracy or softmax confidence) are plotted against entropy. Two variants are commonly used: the \textit{Accuracy Information Curve (AIC)}, which plots classification accuracy across entropy bins, and the \textit{Softmax Information Curve (SIC)}, which plots softmax confidence relative to the original image. Another widely used approach is RISE~\cite{petsiuk2018rise}, which introduces \textit{Insertion} and \textit{Deletion} metrics. In the \textit{Insertion} metric, a fully masked image is progressively restored by adding pixels in saliency order, while tracking the increase in model confidence. In contrast, the \textit{Deletion} metric starts with the full image and removes the most important pixels, observing how quickly the model confidence drops. Strong attribution maps result in faster confidence increases (insertion) and sharper decreases (deletion).

\textbf{Performance Information Curves (PICs).}\cite{kapishnikov2019xrai} evaluate saliency by progressively reintroducing pixels to a blurred image in saliency order. At each step, image information is estimated using compression-based entropy (e.g., WebP size), and model performance is recorded. Two main variants are used: \textit{Accuracy Information Curve (AIC)}, which plots classification accuracy, and \textit{Softmax Information Curve (SIC)}, which tracks confidence. RISE\cite{petsiuk2018rise} introduces \textit{Insertion} and \textit{Deletion} metrics that measure how confidence changes as pixels are added or removed based on attribution. Better saliency maps lead to faster confidence gain in insertion and sharper drop in deletion. 

\textbf{Evaluating Attribution in Weakly Supervised Setting in WSI Classification.} Standard methods like PIC and RISE may not fully apply to WSI settings, where only slide-level labels are available. In Multiple Instance Learning (MIL), a slide is labeled positive if at least one of its patches is positive. Predictions can change abruptly after observing just a few key regions, challenging the common assumption of gradual or proportional prediction shifts. 

To address the limitations of standard PIC in this context, we adapt the evaluation by redefining the x-axis. We start with \textit{control features}, which are patch features sampled from slides of the opposite class (or classes), for example, using non-tumor slide patches when evaluating a tumor case. These typically yield low-confidence predictions for the current slide. We then progressively replace them with features from the target slide, ranked by attribution scores, introducing the most salient features first. 

We structure the evaluation using a set of \textit{information-level bins} defined by two complementary strategies: \textit{Top-$k$ patches:} We evaluate a range of \textit{top-$k$} most salient patches (e.g., $k = 1, 2, \dots, 10, 15, \dots, 500$), which captures how quickly the model's prediction changes as the most informative regions are introduced. And \textit{Saliency thresholds:} We assess prediction behavior as increasing portions of the image are revealed (e.g., top 20\%–99\% of patches), capturing completeness and late-stage transitions. These two strategies are merged into a single ordinal sequence, where each step reflects the order in which salient information is introduced. This design is well-suited to MIL-based WSI classification, as it captures both early, decisive shifts and gradual transitions.

We adapt the standard PIC metric to better suit weakly supervised WSI classification by introducing two specialized evaluation curves, using the \textit{information-level bins} described above. The \textbf{MIL-Accuracy Information Curve (MIL-AIC)} tracks whether the model predicts the correct slide-level label (e.g., tumor vs. non-tumor) as more high-saliency patches are gradually introduced. It measures classification accuracy at each bin, where the output is binary (correct or incorrect), helping assess how quickly the model can make the right decision. The \textbf{MIL-Softmax Information Curve (MIL-SIC)} measures the model’s softmax confidence in the correct class as informative patches are revealed, providing a finer-grained view of the model’s certainty over time. Together, MIL-AIC and MIL-SIC offer a robust and interpretable way to evaluate attribution quality in the MIL setting, where only weak slide-level supervision is available. 

% We refer to this hybrid framework as an adaptation of PIC for weakly supervised learning. Based on this setup, we define two evaluation metrics:
% \begin{itemize}
%     \item \textbf{MIL-Accuracy Information Curve (MIL-AIC):} Plots classification accuracy over the information-level bins.
%     \item \textbf{MIL-Softmax Information Curve (MIL-SIC):} Tracks predicted class confidence relative to the original image.
% \end{itemize}
% Together, these metrics provide a more robust and interpretable evaluation of attribution quality under the MIL setting used in WSI classification. 

\noindent \textbf{Note on applicability.} This evaluation is best suited for \textit{tumor-positive} slides, where predictions shift sharply toward the tumor class after a few key patches are added. By starting from control (normal) features and progressively introducing tumor patches ranked by saliency, the setup effectively captures how well attribution highlights class-discriminative regions. In contrast, normal slides begin with tumor-like features and require removing nearly all of them before the prediction changes, making metrics like AUC less meaningful. Therefore, we focus our evaluation on tumor-positive cases to ensure consistency and interpretability.

\section{Experiments and Results}
\subsection{Dataset}

We evaluate our method on three publicly available pathology datasets: \textbf{CAMELYON16}, \textbf{TCGA-Renal}, and \textbf{TCGA-Lung}. The CAMELYON16 dataset focuses on lymph node sections for breast cancer metastasis detection, containing slides labeled as either \textbf{Tumor} (metastatic tissue) or \textbf{Non-Tumor} (benign tissue). The TCGA-Renal dataset includes three renal cancer subtypes: \textbf{KIRC} (Kidney Renal Clear Cell Carcinoma), \textbf{KIRP} (Kidney Renal Papillary Cell Carcinoma), and \textbf{KICH} (Kidney Chromophobe (Carcinoma)). Similarly, the TCGA-Lung dataset covers two major types of lung cancer: \textbf{LUAD} (Lung Adenocarcinoma) and \textbf{LUSC} (Lung Squamous Cell Carcinoma). These datasets provide diverse diagnostic settings to assess the generalizability of our evaluation framework.
\subsection{Experiment Setting}  
 \begin{table}[!htbp]
    \centering
    \renewcommand{\arraystretch}{1.2}
    \setlength{\tabcolsep}{4pt}
    \begin{tabular}{c|l|cc}
        \toprule
        \textbf{Classifier} & \textbf{Method} & \textbf{MIL-AIC}~↑ & \textbf{MIL-SIC}~↑ \\
        \midrule
        \multicolumn{2}{c|}{Random} & 0.356 ± 0.198 & 0.365 ± 0.199 \\
        \midrule
        \multirow{5}{*}{CLAM}
            & Gradient & 0.893 ± 0.255 & 0.898 ± 0.240 \\
            & IG\cite{IG} & 0.891 ± 0.261 & 0.896 ± 0.243 \\
            & IDG\cite{idg_walker_integrated_2023} & 0.737 ± 0.327 & 0.739 ± 0.322 \\
            & EG \cite{EG} & 0.891 ± 0.255 & 0.895 ± 0.242 \\
            & \textbf{CIG (ours)} & \textbf{0.950 ± 0.166} & \textbf{0.945 ± 0.128} \\
        % \midrule
        % \multicolumn{2}{c|}{\textbf{Random}} & --- & --- \\
        \midrule
        \multirow{5}{*}{MLP}
            & Gradient & 0.961 ± 0.146 & 0.913 ± 0.126 \\
            & IG\cite{IG} & 0.962 ± 0.146 & 0.913 ± 0.127 \\
            & IDG \cite{idg} & 0.850 ± 0.234 & 0.810 ± 0.191 \\
            & EG \cite{EG}  & 0.962 ± 0.146 & 0.913 ± 0.127 \\
            & \textbf{CIG (ours)} & \textbf{0.965 ± 0.128} & \textbf{0.913 ± 0.130} \\
        \bottomrule
    \end{tabular}
    \caption{Attribution performance on \textbf{tumor-positive} slides from the \textbf{Camelyon16} dataset, evaluated using MIL-AIC and MIL-SIC metrics}
    % \caption{Attribution results on the \textbf{Tumor class} for Camelyon16.}
    \label{tab:camelyon16_tumor}
\end{table}

\textbf{Classification Models.} To assess how attribution varies across architectures, we evaluate two MIL models trained to classify both binary and tumor subtypes from WSIs. Gradient-based attribution scores are computed from each model’s output. We compare: (1) a simple MLP bag classifier (\textit{MLP}) with three fully connected layers (512, 256, 128), ReLU activations, dropout, attention pooling, and a final classification layer; and (2) \textit{CLAM}~\cite{clam_lu2021data}, a widely used attention-based MIL model. Both models use patch-level features extracted with a pre-trained ResNet-50 and follow the same preprocessing steps as in~\cite{clam_lu2021data}. Each model is trained for 200 epochs, and final slide-level predictions are used for attribution. For Camelyon16, we use the provided train/test split, while for TCGA-Renal and TCGA-Lung, we create train/val/test splits. All datasets use patient-level separation to prevent data leakage. In the supplementary materials, we evaluate CLAM with CONCH~\cite{lu2024visual} features in place of ResNet-50 to examine the effect of stronger pathology representations on attribution.

\begin{table*}[h]
    \centering
    \renewcommand{\arraystretch}{1.2}
    \setlength{\tabcolsep}{4pt}
    \begin{tabular}{c|l|cc|cc|cc}
        \toprule
        \multirow{2}{*}{\textbf{Classifier}} & \multirow{2}{*}{\textbf{Method}} & \multicolumn{2}{c|}{\textbf{pRCC}} & \multicolumn{2}{c|}{\textbf{ccRCC}} & \multicolumn{2}{c}{\textbf{chRCC}} \\
        \cmidrule(lr){3-4} \cmidrule(lr){5-6} \cmidrule(lr){7-8}
        & & MIL-AIC~↑ & MIL-SIC~↑ & MIL-AIC~↑ & MIL-SIC~↑ & MIL-AIC~↑ & MIL-SIC~↑ \\
        \midrule
        \multicolumn{2}{c|}{Random} & 0.092 ± 0.037 & 0.093 ± 0.036 & 0.276 ± 0.155 & 0.276 ± 0.155 & 0.088 ± 0.033 & 0.091 ± 0.033 \\
        \midrule
        \multirow{5}{*}{CLAM}
            & Gradient & 0.096 ± 0.031 & 0.098 ± 0.030 & 0.730 ± 0.379 & 0.740 ± 0.364 & 0.100 ± 0.029 & 0.099 ± 0.027 \\
            & IG\cite{IG} & 0.093 ± 0.035 & 0.095 ± 0.034 & 0.729 ± 0.377 & 0.739 ± 0.363 & 0.092 ± 0.036 & 0.093 ± 0.034 \\
            & IDG\cite{idg_walker_integrated_2023}  & 0.196 ± 0.040 & 0.196 ± 0.039 & 0.635 ± 0.284 & 0.638 ± 0.277 & 0.198 ± 0.031 & 0.198 ± 0.030 \\
            & EG\cite{EG}  & 0.097 ± 0.034 & 0.098 ± 0.033 & 0.730 ± 0.378 & 0.740 ± 0.363 & 0.094 ± 0.030 & 0.094 ± 0.029 \\
            & \textbf{CIG (ours)} & \textbf{0.204 ± 0.044} & \textbf{0.206 ± 0.043} & \textbf{0.776 ± 0.297} & \textbf{0.783 ± 0.286} & \textbf{0.209 ± 0.035} & \textbf{0.210 ± 0.035} \\
        % \midrule
        % \multicolumn{2}{c|}{\textbf{Random}} & --- & --- & --- & --- & --- & --- \\
        \midrule
        \multirow{5}{*}{MLP}
            & Gradient & 0.122 ± 0.049 & 0.126 ± 0.053 & 0.535 ± 0.357 & 0.530 ± 0.333 & 0.110 ± 0.041 & 0.113 ± 0.033 \\
            & IG\cite{IG} & 0.116 ± 0.051 & 0.121 ± 0.055 & 0.536 ± 0.359 & 0.531 ± 0.335 & 0.113 ± 0.029 & 0.117 ± 0.028 \\
            & IDG\cite{idg_walker_integrated_2023}  & 0.218 ± 0.051 & 0.221 ± 0.056 & 0.442 ± 0.233 & 0.455 ± 0.222 & 0.207 ± 0.018 & 0.207 ± 0.015 \\
            & EG \cite{EG}  & 0.119 ± 0.047 & 0.124 ± 0.053 & 0.535 ± 0.360 & 0.531 ± 0.334 & 0.114 ± 0.035 & 0.115 ± 0.034 \\
            & \textbf{CIG (ours)} & \textbf{0.226 ± 0.059} & \textbf{0.228 ± 0.060} & \textbf{0.547 ± 0.293} & \textbf{0.557 ± 0.270} & \textbf{0.221 ± 0.021} & \textbf{0.221 ± 0.017} \\
        \bottomrule
    \end{tabular}
    \caption{Attribution performance on each class from the \textbf{TCGA-RCC} dataset, evaluated using MIL-AIC and MIL-SIC metrics.}
    % \caption{Attribution results for each class in \textbf{TCGA-RCC} .}
    \label{tab:tcga_rcc_classes}
\end{table*}
\begin{table*}[h]
    \centering
    \renewcommand{\arraystretch}{1.2}
    \setlength{\tabcolsep}{4pt}
    \begin{tabular}{c|l|cc|cc}
        \toprule
        \multirow{2}{*}{\textbf{Classifier}} & \multirow{2}{*}{\textbf{Method}} & \multicolumn{2}{c|}{\textbf{LUAD}} & \multicolumn{2}{c}{\textbf{LUSC}} \\
        \cmidrule(lr){3-4} \cmidrule(lr){5-6}
        & & MIL-AIC~↑ & MIL-SIC~↑ & MIL-AIC~↑ & MIL-SIC~↑ \\
        \midrule
        \multicolumn{2}{c|}{Random} & 0.147 ± 0.037 & 0.148 ± 0.035 & 0.200 ± 0.070 & 0.200 ± 0.068 \\
        \midrule
        \multirow{5}{*}{CLAM}
            & Gradient & 0.165 ± 0.129 & 0.167 ± 0.122 & 0.396 ± 0.357 & 0.405 ± 0.352 \\
            & IG\cite{IG}  & 0.162 ± 0.131 & 0.164 ± 0.124 & 0.396 ± 0.355 & 0.405 ± 0.352 \\
            & IDG\cite{idg_walker_integrated_2023}  & 0.261 ± 0.098 & 0.265 ± 0.095 & 0.424 ± 0.204 & 0.427 ± 0.199 \\
            & EG \cite{EG} & 0.164 ± 0.131 & 0.167 ± 0.125 & 0.395 ± 0.356 & 0.404 ± 0.352 \\
            & \textbf{CIG (ours)} & \textbf{0.312 ± 0.184} & \textbf{0.315 ± 0.184} & \textbf{0.603 ± 0.315} & \textbf{0.607 ± 0.308} \\
        % \midrule
        % \multicolumn{2}{c|}{\textbf{Random}} & --- & --- & --- & --- \\
        \midrule
        \multirow{5}{*}{MLP}
            & Gradient & 0.140 ± 0.043 & 0.144 ± 0.042 & 0.676 ± 0.383 & 0.684 ± 0.351 \\
            & IG \cite{IG} & 0.139 ± 0.039 & 0.145 ± 0.038 & 0.672 ± 0.381 & 0.680 ± 0.351 \\
            & IDG \cite{idg_walker_integrated_2023} & \textbf{0.302 ± 0.107} & \textbf{0.340 ± 0.112} & 0.644 ± 0.311 & 0.656 ± 0.298 \\
            & EG \cite{EG} & 0.141 ± 0.036 & 0.147 ± 0.036 & 0.672 ± 0.382 & 0.680 ± 0.351 \\
            & \textbf{CIG (ours)} & 0.281 ± 0.113 & 0.303 ± 0.113 & \textbf{0.759 ± 0.296} & \textbf{0.765 ± 0.277} \\
        \bottomrule
    \end{tabular}
    \caption{Attribution performance on each class from the \textbf{TCGA-Lung} dataset, evaluated using MIL-AIC and MIL-SIC metrics.}
    % \caption{Attribution results for each class in \textbf{TCGA-Lung}.}
    \label{tab:tcga_lung_classes}
\end{table*}

\textbf{Methods under comparison.} For interpretability, we benchmark several gradient-based attribution methods Vanilla Gradient, Integrated Gradients (IG) \cite{IG}, Expected Gradients (EG) \cite{EG}, Integrated Decision Gradients (IDG) \cite{idg_walker_integrated_2023}, and our proposed Contrastive Integrated Gradients (CIG) as well as a \textit{Random} baseline.  
All path-based methods (IG, EG, IDG, CIG) use 50 interpolation steps. The \textit{Random} denote random baseline assigns uniform scores to patches, and its MIL-AIC and MIL-SIC metrics are computed using the CLAM outputs. To generate saliency maps, we compute a per-patch attribution score by averaging the absolute values across all feature dimensions, producing a single importance value that is mapped back to the patch’s spatial location. Using absolute values captures both positive and negative influences while the magnitude reflects the strength of the model’s response, highlighting the most impactful regions.

\textbf{Information-level bins.} To evaluate how model predictions evolve as informative regions are added, we use two complementary axes: \textit{top-$k$} and \textit{saliency thresholds}. The \textit{top-$k$} axis captures early-stage changes using a dense-to-sparse sequence: $k = 1$–10, 15–50 (step size 5), 60–100 (step size 10), and then 150, 200, 300, 400, 500. This helps assess how quickly predictions change as highly salient patches are introduced. The \textit{saliency threshold} axis uses percentile cutoffs (e.g., 20\%, 40\%, 60\%, 80\%, 85\%, 90\%, 95\%, 99\%) to evaluate later-stage transitions and completeness as more of the image is revealed.

% \textbf{Information-level bins.} To evaluate how model confidence evolves as informative regions are introduced, we define two complementary evaluation axes: \textit{top-$k$} and \textit{saliency thresholds}. The \textit{top-$k$} setting captures how quickly the model’s prediction changes when only the most salient patches are introduced, with $k$ ranging densely from 1 to 10 and more sparsely up to 500. This is particularly useful for assessing cases where the model is highly confident with only a few key regions. In contrast, the \textit{saliency threshold} setting selects patches based on a percentile cutoff of attribution scores, e.g., the top 20\% to 99\% providing a more complete view of prediction dynamics, especially when the model requires more context to make a decision.  
 
% \textbf{Baseline Construction.} For path-based attribution methods, we construct contrastive baselines by sampling features randomly from 30 slides of the opposite class. For each target slide, the baseline pool is formed by randomly selecting the same number of patches from each reference slide and aggregating them. All IG-based methods share the same baseline per input slide. Vanilla Gradient does not require a baseline.

\textbf{Baseline Construction.} For path-based attribution methods, we construct contrastive baselines by sampling patch features from 30 slides of the opposite class. For each target slide, an equal number of patches are randomly selected from each reference slide and aggregated to form a shared baseline pool. All IG-based methods use the same baseline per input. Vanilla Gradient does not require a baseline.

% \textbf{Evaluation Metrics. }Performance is evaluated using MIL-AIC and MIL-SIC metrics, where higher values indicate better class-discriminative and instance-consistent attributions.   Higher scores indicate stronger alignment between attribution maps and decision-relevant regions. 
\textbf{Evaluation Metrics.} We evaluate performance using MIL-AIC and MIL-SIC, where higher scores reflect better class discrimination and instance consistency, indicating stronger alignment between attribution maps and decision-relevant regions.
\subsection{Quantitative Results} 
\begin{figure*}[h]
    \centering
    \includegraphics[width=1\textwidth]{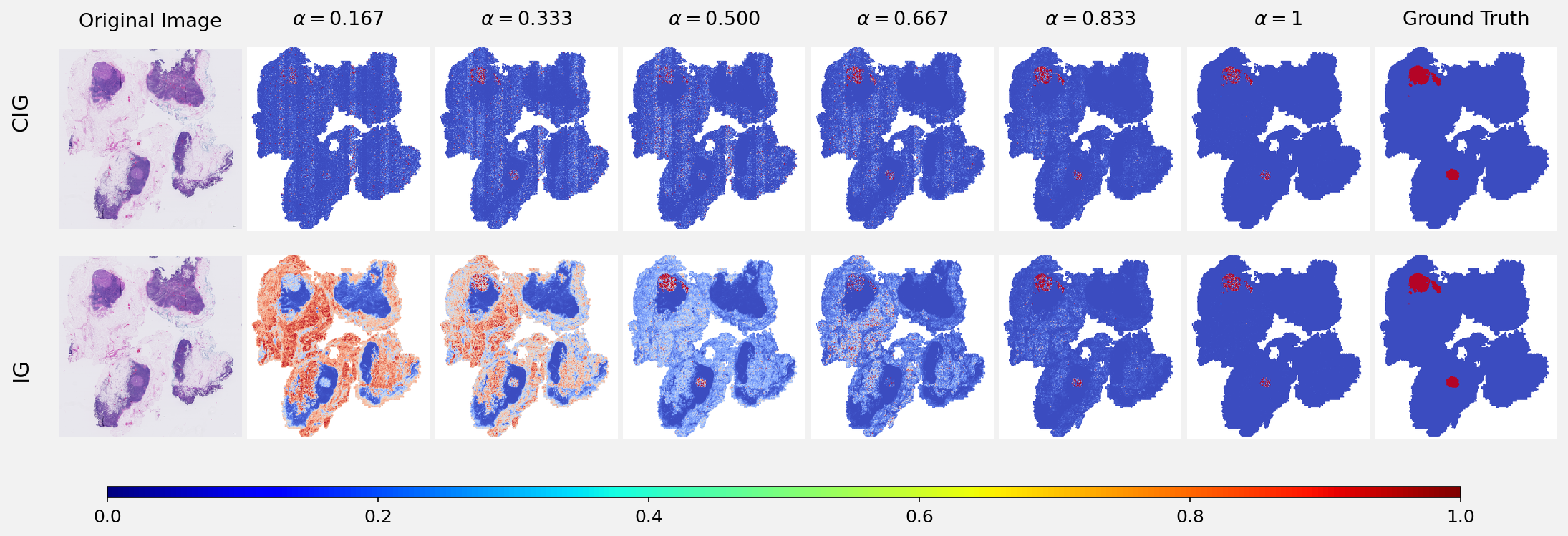}
    \caption{Comparison of Integrated Gradients (IG) and Contrastive Integrated Gradients (CIG) across interpolation steps (\( \alpha \)), each row shows intermediate gradient maps at increasing \( \alpha \) values, from 0.167 to 1.0, illustrating how gradients evolve along the interpolation path. Note that the final heatmap (\( \alpha = 1 \)) shows only the gradient at the last step and is not the complete attribution result. The full attribution is computed by summing the gradients across all \( \alpha \) values and multiplying by the input difference \( x' - x \). CIG produces more stable and localized gradients in tumor regions throughout the path, while IG exhibits more dispersed patterns.}
    \label{fig:contrs_vs_grad}
\end{figure*} 

\begin{figure*}[h]
    \centering
    \includegraphics[width=1\textwidth]{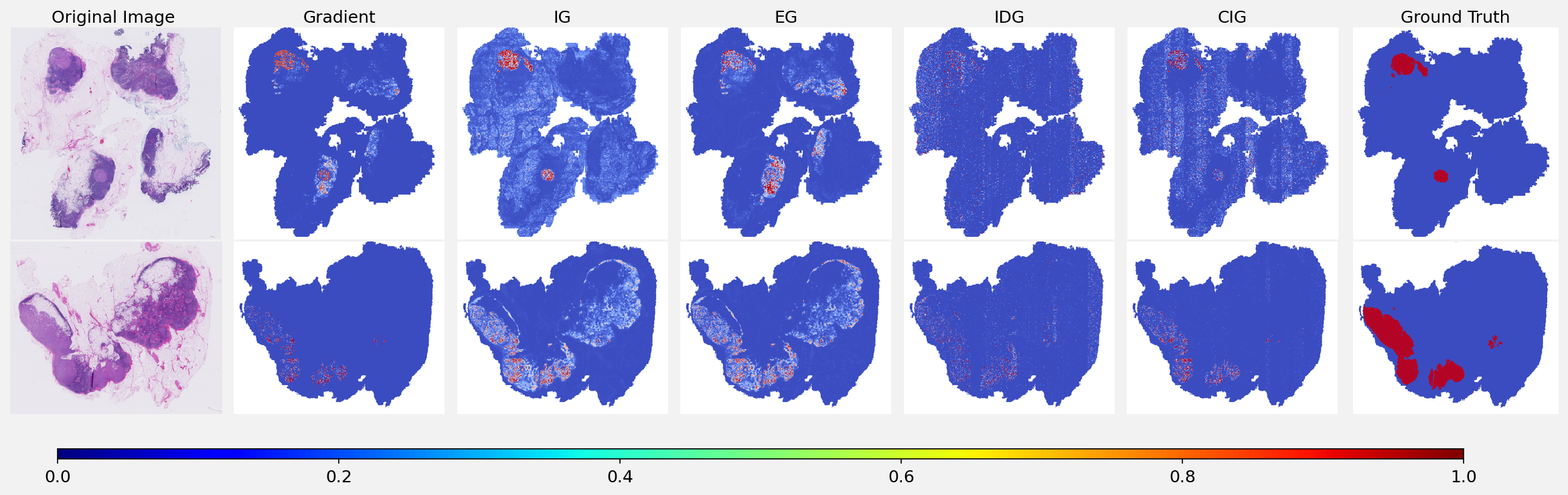}  % Adjust width as needed
    \caption{Qualitative comparison of attribution maps generated by five saliency methods (Vanilla Gradient, IG\cite{IG}, EG \cite{EG}, IDG\cite{idg_walker_integrated_2023}, CIG) on Camelyon16 tumor slides. The first column shows the original WSI patch, and the last column shows the ground truth mask. CIG produces more focused and accurate attributions that align closely with tumor regions.}
    \label{fig:qualitative_attribution}
\end{figure*}   
As shown in Table~\ref{tab:camelyon16_tumor}, CIG achieves the highest performance across both classifiers in terms of MIL-AIC and MIL-SIC. For CLAM, it scores \textit{0.950 ± 0.166} (MIL-AIC) and \textit{0.945 ± 0.128} (MIL-SIC); for MLP, it reaches \textit{0.965 ± 0.128} and \textit{0.913 ± 0.130}, respectively. These results indicate that CIG more effectively highlights decision-relevant tumor regions.

Table \ref{tab:tcga_rcc_classes} reports attribution results on the TCGA-RCC dataset across three renal subtypes. CIG generally achieves strong performance across both models. For pRCC, CIG with the MLP model reaches \textit{0.226 ± 0.059} (MIL-AIC) and \textit{0.228 ± 0.060} (MIL-SIC). For ccRCC, CIG with CLAM performs well with \textit{0.776 ± 0.297} (MIL-AIC) and \textit{0.783 ± 0.286} (MIL-SIC). These results suggest CIG works reliably across subtypes and model types.

Table \ref{tab:tcga_lung_classes} shows attribution results on the TCGA-Lung dataset for LUAD and LUSC subtypes. For LUSC, CIG achieves the highest scores with CLAM (\textit{0.759 ± 0.296} MIL-AIC, \textit{0.765 ± 0.277} MIL-SIC). In LUAD, while IDG performs well with the MLP model, CIG offers a strong balance between consistency and class relevance.

Across all datasets and models, CIG generally achieves strong attribution quality, often ranking among the top-performing methods. These results suggest that combining contrastive baselines with path-based sensitivity can enhance the consistency and relevance of feature attributions. CIG appears effective at identifying important regions across different classifier types and cancer subtypes, supporting its broader applicability in weakly supervised WSI settings.
 To better understand attribution behavior, we present a qualitative analysis on selected Camelyon16 tumor slides. We compare CIG with Vanilla Gradient, IG~\cite{IG}, EG~\cite{EG}, and IDG~\cite{idg_walker_integrated_2023}, and also examine how gradients evolve across interpolation steps in IG and CIG.

\subsection{Qualitative Results}

% Figures~\ref{fig:contrs_vs_grad} and~\ref{fig:qualitative_attribution} highlight key distinctions between methods. Figure~\ref{fig:contrs_vs_grad} shows intermediate gradients from IG and CIG across interpolation steps \( \alpha \) (ranging from 0.167 to 1.0), illustrating how sensitivity evolves between input and baseline. These are not final attributions but stepwise gradients used in the integration. CIG gradients remain more localized and stable in tumor regions throughout the path, while IG gradients appear more dispersed, potentially reflecting sensitivity to visually prominent but less predictive features.

Figures~\ref{fig:contrs_vs_grad} and~\ref{fig:qualitative_attribution} highlight key distinctions in attribution behavior. Figure~\ref{fig:contrs_vs_grad} presents intermediate gradient maps from IG and CIG across interpolation steps \( \alpha \) (ranging from 0.167 to 1.0), illustrating how sensitivity evolves along the path from baseline to input. These maps reflect individual gradient snapshots rather than the final integrated attribution. Throughout the path, CIG gradients tend to be more localized and consistent in tumor regions. In comparison, IG gradients appear more spatially distributed, potentially reflecting a broader sensitivity to feature variations that may or may not contribute directly to the model's decision.

% Figure~\ref{fig:qualitative_attribution} presents the final attribution maps. Patch-level scores are aggregated and mapped back to slide coordinates for visual comparison with tumor annotations. In some cases, IG and EG highlight off-target regions not included in the ground truth mask, while CIG consistently emphasizes areas that align more closely with tumor regions. Together, these visualizations support the conclusion that CIG provides more focused and decision-relevant attributions, consistent with its improved quantitative performance.

Figure~\ref{fig:qualitative_attribution} shows the final attribution maps for each method. Patch-level importance scores are aggregated and projected back onto slide coordinates for comparison with tumor annotations. In some examples, IG and EG highlight additional regions not included in the ground truth mask, potentially reflecting sensitivity to visually salient but less discriminative features. In contrast, CIG consistently emphasizes areas that closely align with annotated tumor regions. These visual patterns align with CIG’s stronger quantitative results, suggesting it offers more focused and decision-relevant attributions.

\section{Conclusion and Future Work}

% In summary, we introduce Contrastive Integrated Gradients, a feature attribution method that enhances interpretability by computing contrastive gradients in logit space. CIG identifies class-discriminative regions and satisfies integrated attribution axioms, and we further propose two evaluation metrics, MIL-AIC and MIL-SIC, for assessing attribution quality in weakly supervised settings. CIG consistently produces more informative and stable attributions than existing methods, both quantitatively and qualitatively, across three cancer datasets. By improving interpretability, CIG contributes to more explainable AI systems, supporting more efficient and consistent decision-making in computational pathology. 
In summary, we introduce Contrastive Integrated Gradients (CIG), a feature attribution method that enhances interpretability by computing contrastive gradients in logit space. CIG identifies class-discriminative regions, satisfies integrated attribution axioms, and is evaluated with two new metrics, MIL-AIC and MIL-SIC, for weakly supervised settings. Across three cancer datasets, CIG consistently produces more informative and stable attributions than existing methods, both quantitatively and qualitatively. By improving interpretability, CIG supports more explainable AI systems and more consistent decision-making in computational pathology. As future work, we plan to incorporate rigorous human-subject evaluations of interpretability.

\section{Acknowledgments.}
 This work was supported by the National Institutes of Health (NIH) under Grant 5R01DK134055-02.  
%%%%%%%%% REFERENCES
{\small
\bibliographystyle{ieee_fullname}
\bibliography{ref}
}

\end{document}